%% file: sample_gaussian_arxiv.tex
\documentclass[11pt]{article}
\usepackage{amsfonts}\pdfoutput=1
\usepackage{fullpage}
\usepackage{times}
\usepackage{graphicx,color}
\usepackage{array,float}
\usepackage{url}
\usepackage[usenames,dvipsnames]{xcolor}
\usepackage{amstext,amssymb,amsmath}
\usepackage{amsthm}
\usepackage{mathtools}
\usepackage{verbatim}
\usepackage{bm}
\usepackage{paralist}
\usepackage{mathrsfs}
\usepackage{ulem}
\usepackage[numbers]{natbib}
\usepackage{hyperref}
\usepackage[noend]{algorithmic}
\usepackage{algorithm}
\usepackage{booktabs}
\usepackage{natbib}
\usepackage{subcaption}

\DeclarePairedDelimiterX{\infdivx}[2]{(}{)}{#1\;\delimsize\|\;#2}
\newcommand{\rd}[3]{\mathrm{D}_{#1}\infdivx*{#2}{#3}}
\newcommand{\rdalpha}[2]{\mathrm{D}_{\alpha}\infdivx*{#1}{#2}}

\newtheorem{lem}{Lemma}

\newtheorem{thm}[lem]{Theorem}
\newtheorem{cor}[lem]{Corollary}

\newtheorem{defi}[lem]{Definition}

\newcommand{\eps}{\varepsilon}

\newcommand{\M}{\mathcal{M}}  % mechanism
\newcommand{\R}{\mathbb{R}}
\newcommand{\eqdef}{\stackrel{\Delta}{=}}
\newcommand{\E}{\mathbb{E}}
\newcommand{\erfc}{\mathrm{erfc}}
\newcommand{\SG}{\mathrm{SG}}

\newcommand{\calN}{\ensuremath{\mathcal N}}

\newcommand{\ud}{\,\mbox{d}}

\newcommand{\myiff}{\ensuremath{\Leftrightarrow}}

\newcommand{\renyi}{R\'enyi}
\begin{document}
\title{R{\'e}nyi Differential Privacy of the Sampled Gaussian Mechanism}

\author{Ilya Mironov\thanks{
Google Research. \texttt{mironov@google.com}}
\and Kunal Talwar\thanks{Google Research. \texttt{kunal@google.com}}
\and Li Zhang\thanks{Google Research. \texttt{liqzhang@google.com}}}  %optional

\date{}

%\author{Ilya Mironov \and Kunal Talwar \and Li Zhang}
%\date{Google Research\\
%	\{\texttt{mironov,kunal,liqzhang}\}\texttt{@google.com}}
\maketitle

\begin{abstract}
The Sampled Gaussian Mechanism (SGM)---a composition of subsampling and the additive Gaussian noise---has been successfully used in a number of machine learning applications. The mechanism's unexpected power is derived from privacy amplification by sampling where the privacy cost of a single evaluation diminishes quadratically, rather than linearly, with the sampling rate. Characterizing the precise privacy properties of SGM motivated development of several relaxations of the notion of differential privacy.

This work unifies and fills in gaps in published results on SGM. We describe a numerically stable procedure for precise computation of SGM's  \renyi\ Differential Privacy and prove a nearly tight (within a small constant factor) closed-form bound.
\end{abstract}

\input{introduction}

\input{reduction}

\input{computation}

\input{discussion}

\bibliographystyle{alpha}
\bibliography{sample_gaussian}

\appendix
\end{document}

%% file: introduction.tex
\section{Introduction and Definitions}

The notion of differential privacy~\cite{Dwork-ICALP, DMNS} grounds a guarantee of individual privacy for input of a statistical procedure in non-determinism of the procedure's output. The uncertainty, or randomness, can be achieved either explicitly, by injecting noise, or implicitly, by leveraging randomness intrinsic to the mechanism or its input. Developing a general framework that accounts for all sources of output stochasticity and distills them into a guarantee of differential privacy remains a significant open problem.

In this paper we study privacy of an important and widely applicable algorithm, Sampled Gaussian mechanism (SGM). The mechanism combines two very natural, and common in statistics, sources of uncertainty: sampling and additive Gaussian noise. SGM operates by drawing a random subset from a larger dataset, followed by application of a function with output in $\mathbb{R}^d$, and addition of $d$-dimensional spherical Gaussian noise. Both components of SGM are well-known from the literature on differential privacy. Sampling, as a building block of differentially private mechanisms, was first analyzed by Kasiviswanathan et al.~\cite{KLNRS,Smith-blog}. The additive Gaussian noise as an alternative to Laplace noise was discussed by Dwork et al.~\cite{ODO}. Their joint application, or SGM, was proposed by Abadi et al.~\cite{DP-DL} as a subroutine in their implementation of differentially private stochastic gradient descent.

Following introduction of (pure) differential privacy~\cite{DMNS}, many of its relaxations, which we recall shortly, were motivated by the need to capture properties of the Gaussian additive noise and SGM.

The notion of approximate differential privacy, which includes an additive $\delta$ term, appeared in the work by Dwork et al.~\cite{ODO} in order to support analysis of the Gaussian noise mechanism. Concentrated Differential Privacy (CDP)~\cite{DR16-CDP} and its reformulation, zero-CDP, due to Bun and Steinke~\cite{zCDP} were developed to refine composition theorems using the Gaussian mechanism as their main motivating examples.

Another family of definitions and analytical techniques emerged as tools for handling the Sampled Gaussian mechanism. To track privacy loss budget over multiple applications of SGM, Abadi et al.~developed a \emph{moments accountant}, implemented a numerically stable and efficient algorithm for computing it, and analyzed the accountant's asymptotic properties. A relaxation of CDP well-suited for analysis of SGM, called truncated CDP (tCDP), was introduced by Bun et al.~~\cite{tCDP}. We compare these closely related notions in Section~\ref{s:discussion}.

{\sloppy % due to overflowing url
In this paper we revisit the moments accountant due to Abadi et al., restate it using the notion of \renyi\ differential privacy (RDP)~\cite{RDP}, relax certain assumptions, and strengthen upper bounds. An open-source implementation of the RDP accountant for SGM is available from \href{https://github.com/tensorflow/privacy/blob/master/privacy/analysis/rdp_accountant.py}{https://github.com/tensorflow/privacy}.

%\cite{WBK18-subsampled}

We recall the definitions of \renyi\ divergence, \renyi\ differential privacy and the Sampled Gaussian mechanism.
}

\begin{defi}[\renyi\ divergence] Let $P$ and $Q$ be two distributions on $\mathcal{X}$ defined over the same probability space, and let $p$ and $q$ be their respective densities. The \renyi\ divergence of a finite order $\alpha\neq 1$ between $P$ and $Q$ is defined as
\[
\rdalpha{P}{Q}\eqdef \frac1{\alpha-1}\ln \int_{\mathcal{X}} q(x)\left(\frac{p(x)}{q(x)}\right)^{\alpha}\,\mathrm{d}x.
\]
\renyi\ divergence at orders $\alpha=1,\infty$ are defined by continuity.
\end{defi}
We refer the reader to van Erven and Harremo\"es~\cite{vEH14-Renyi} for an introduction to the theory of \renyi\ divergence, including proof of its continuity and many other useful properties.

\begin{defi}[\renyi\ differential privacy (RDP)] We say that a randomized mechanism $\mathcal{M}\colon \mathcal{S}\to\mathcal{R}$ satisfies $(\alpha,\eps)$-\renyi\ differential privacy (RDP) if for any two \emph{adjacent} inputs $S,S'\in \mathcal{S}$ it holds that
\[
\rdalpha{\mathcal{M}(S)}{\mathcal{M}(S')}\leq \eps.
\]
\end{defi}
The notion of adjacency between input datasets is domain-specific and is usually taken to mean that the inputs differ in contributions of a single individual. In this work, we will use this definition and call two datasets $S, S'$ to be adjacent if $S' = S \cup \{x\}$ for some $x$ (or vice versa).

\begin{defi}[Sampled Gaussian Mechanism (SGM)] Let $f$ be a function mapping subsets of $\mathcal{S}$ to $\mathbb{R}^d$. We define the Sampled Gaussian mechanism (SGM) parameterized with the sampling rate $0<q\leq 1$ and the noise $\sigma>0$ as
	\[
	\SG_{q,\sigma}(S)\eqdef f(\{x\colon x\in S \textrm{ is sampled with probability } q\})+\mathcal{N}(0,\sigma^2\mathbb{I}^d),
	\]
where each element of $S$ is sampled independently at random with probability $q$ without replacement, and $\mathcal{N}(0,\sigma^2\mathbb{I}^d)$ is spherical $d$-dimensional Gaussian noise with per-coordinate variance $\sigma^2$.
\end{defi}

The SGM has been studied in several incomparable settings. Broadly speaking, there are three approaches for deriving privacy bounds on the SGM: (1) asymptotic analyses that target tight bounds for $q\to 0$; (2) approaches that lead to numerically accurate estimates; and (3) a closed-form analysis that captures the right dependency on the order $\alpha$. These results, along with our contributions are summarized in Table~\ref{tab:knowledge}.

\begin{table}
	\begin{center}
		\small
	\begin{tabular}{lll}
		\toprule
		Reference & Conditions & Privacy bound \\ \midrule
		Abadi et al.\cite{DP-DL} & $\begin{aligned}
		\textstyle q&\textstyle<\frac1{16\sigma}\\
		\textstyle \alpha&\textstyle \leq 1+\sigma^2\ln\frac1{q\sigma}\\
		\end{aligned}$ &
		$(\alpha, q^2\frac{\alpha}{(1-q)\sigma^2}+O(q^3\alpha^3/\sigma^3))$-RDP for $q\to 0$\\  \midrule
		Abadi et al.\cite{DP-DL} & integer $\alpha$& Numerical procedure\\ \midrule
		Bun et al.~\cite{tCDP} & $\begin{aligned}
		\textstyle q&\textstyle\leq \frac1{10},\sigma\geq \sqrt{5}\\
		\textstyle\alpha&\textstyle\leq \frac12\sigma^2\ln{\frac1q}\\
		\end{aligned}$&$(\alpha,q^2\cdot\frac{6\alpha}{\sigma^2})$-RDP for fixed-size sample\\ \midrule
		% Wang et al.~\cite{pmlr-wang19b} & integer $\alpha$ & \\  \midrule
		this work & $\begin{aligned}
		\textstyle q &\textstyle< \frac1{5}, \sigma \geq 4\\
		\textstyle \alpha&\leq {\textstyle\frac12}\sigma^2L-2\ln \sigma\\
		\alpha&\leq \frac{\frac12\sigma^2 L^2\!-\!\ln5\!-\!2\ln\sigma}{L+\ln (q\alpha)+\frac1{2\sigma^2}}
		\end{aligned}$ &
$(\alpha, q^2\cdot \frac{2\alpha}{\sigma^2})$-RDP for i.i.d.\ (Poisson) sample\\  \midrule
this work & arbitrary $\alpha\geq 1$& Numerical procedure\\
		\bottomrule
	\end{tabular}
\end{center}
\caption{\small Summary of various approaches to bounding privacy of the Sampled Gaussian mechanism. The mechanism is applied to a function of $\ell_2$-sensitivity 1. Bun et al.\ analyze a version of the mechanism where the sample has fixed size $qN$, and two datasets are adjacent if they cover the same population and differ in the \emph{value} (not the presence) of one person's contribution. $L$ is a shortcut for $\ln\left(1+\frac1{q(\alpha-1)}\right)$.}\label{tab:knowledge}
\end{table}

%% file: reduction.tex
\section{Reducing to a simpler case}

We first reduce the problem of proving the RDP bound for the Sampled Gaussian mechanism to a particularly simple special case of a mixture of single-dimensional Gaussians.

\begin{thm}Let $\SG_{q,\sigma}$ be the Sampled Gaussian mechanism for some function $f$. Then $\SG_{q,\sigma}$ satisfies $(\alpha,\eps)$-RDP whenever
\begin{align*}
\eps&\leq \rdalpha{\mathcal{N}(0, \sigma^2)}{(1-q) \mathcal{N}(0, \sigma^2) + q \mathcal{N}(1, \sigma^2)},\\
\textrm{and}\;\;\eps&\leq \rdalpha{(1-q) \mathcal{N}(0, \sigma^2) + q \mathcal{N}(1, \sigma^2)}{\mathcal{N}(0, \sigma^2)},
\end{align*}
under the assumption $\|f(S)-f(S')\|_2 \leq 1$ for any adjacent $S,S'\in\mathcal{S}$.
\end{thm}
\begin{proof}
Let $S, S' \in \mathcal{S}$ be a pair of adjacent datasets such that $S' = S \cup \{x\}$. We wish to bound the \renyi\ divergences $\rdalpha{\M(S)}{\M(S')}$ and $\rdalpha{\M(S')}{\M(S)}$, where $\M$ is the Sampled Gaussian mechanism for some function $f$ with $\ell_2$-sensitivity $1$.

Let $T$ denote a set-valued random variable defined by taking a random subset of $\mathcal{S}$, where each element of $S$ is independently placed in $T$ with probability $q$. Conditioned on $T$, the mechanism $\M(S)$ samples from a Gaussian with mean $f(T)$. Thus
\[
\M(S)= \sum_T p_T \mathcal{N}(f(T), \sigma^2\mathbb{I}^d),
\]
where the sum here denotes mixing of the distributions with the weights $p_T$.
Similarly,
\[
\M(S')= \sum_T p_T \left((1-q)\mathcal{N}(f(T), \sigma^2\mathbb{I}^d) + q\mathcal{N}(f(T\cup\{x\}), \sigma^2\mathbb{I}^d)\right).
\]

\renyi\ divergence is quasi-convex~\cite{vEH14-Renyi}, allowing us to bound
\begin{align*}
\rdalpha{\M(S)}{\M(S')} &\leq \sup_T \rdalpha{\mathcal{N}(f(T), \sigma^2\mathbb{I}^d)}{(1-q)\mathcal{N}(f(T), \sigma^2\mathbb{I}^d) + q\mathcal{N}(f(T\cup\{x\}), \sigma^2\mathbb{I}^d)}\\
&\leq \sup_T \rdalpha{\mathcal{N}(0, \sigma^2\mathbb{I}^d)}{(1-q)\mathcal{N}(0, \sigma^2\mathbb{I}^d) + q\mathcal{N}(f(T\cup\{x\})-f(T), \sigma^2\mathbb{I}^d)},
\end{align*}
where we have used the translation invariance of \renyi\ divergence.
Since the covariances are symmetric, we can, by applying a rotation, assume that $f(T\cup\{x\}) - f(T) = c_T \mathbf{e}_1$ for some constant $c_T \leq 1$.  The two distributions at hand are then both product distributions that are identical in all coordinates except the first. By additivity of \renyi\ divergence for product distributions, we have that
\begin{align*}
\rdalpha{\M(S)}{\M(S')} &\leq \sup_{c \leq 1} \rdalpha{\mathcal{N}(0, \sigma^2)}{(1-q) \mathcal{N}(0, \sigma^2) + q \mathcal{N}(c, \sigma^2)}\\
 &= \sup_{c \leq 1} \rdalpha{\mathcal{N}(0, (\sigma/c)^2)}{(1-q) \mathcal{N}(0, (\sigma/c)^2) + q \mathcal{N}(1, (\sigma/c)^2)}.
\end{align*}
For any $c\leq 1$,  the noise $\mathcal{N}(0, (\sigma/c)^2)$ can be obtained from $\mathcal{N}(0, \sigma^2)$ by adding noise from $\mathcal{N}(0, (\sigma/c)^2 - \sigma^2)$, and the same operation allows us to to obtain $(1-q) \mathcal{N}(0, (\sigma/c)^2) + q \mathcal{N}(1, (\sigma/c)^2)$ from $(1-q) \mathcal{N}(0, \sigma^2) + q \mathcal{N}(1, \sigma^2)$. Thus by the data processing inequality for \renyi\ divergence, we conclude
\[
\rdalpha{\M(S)}{\M(S')} \leq \rdalpha{\mathcal{N}(0, \sigma^2)}{(1-q) \mathcal{N}(0, \sigma^2) + q \mathcal{N}(1, \sigma^2)}.
\]
An identical argument implies that
\[
\rdalpha{\M(S')}{\M(S)} \leq \rdalpha{(1-q) \mathcal{N}(0, \sigma^2) + q \mathcal{N}(1, \sigma^2)}{\mathcal{N}(0, \sigma^2)}
\]
as claimed.
\end{proof}

In the next section, we bound these simpler one-dimensional \renyi\ divergences between mixtures of Gaussian distributions.

%% file: computation.tex
\section{RDP Analysis of Single-Dimensional SGM}

Let $\mu_0$ denote the
pdf of $\calN(0, \sigma^2)$ and let $\mu_1$ denote the pdf of
$\calN(1, \sigma^2)$. Let
\begin{align*}
\mathcal{M}(S) &\sim \mu_0,\\
\mathcal{M}(S') &\sim \mu \eqdef (1-q)\mu_0 + q\mu_1.
\end{align*}
We introduce the following notation used through the rest of this section. Define
\begin{align*}
A_\alpha &\eqdef \E_{z \sim \mu_0} [ (\mu(z) / \mu_0(z))^\alpha]\\
\mbox{and  }{} B_\alpha &\eqdef \E_{z \sim \mu\hphantom{_0}} [ (\mu_0(z) / \mu(z))^\alpha].
\end{align*}
The previous section demonstrates that SGM applied to a function of $\ell_2$-sensitivity 1 satisfies $(\alpha, \eps)$-RDP if $\eps\leq \frac1{\alpha-1}\log\max(A_\alpha,B_\alpha)$. Thus, analyzing RDP properties of SGM is equivalent to upper bounding $A_\alpha$ and $B_\alpha$.

In our first result (Section~\ref{ss:a_geq_b}), we demonstrate that $A_\alpha\geq B_\alpha$. In fact, we prove a more general statement about centrally-symmetric distributions, from which the case of $\mu_0=\calN(0, \sigma^2)$ and $\mu_1=\calN(1, \sigma^2)$ follows as a corollary.

To upper bound $A_\alpha$ we pursue two complementary approaches. Section~\ref{ss:closed-form} derives a closed-form bound that is valid and reasonably tight within a wide range of parameters. Section~\ref{ss:numerical} describes a numerically stable computational procedure for computing $A_\alpha$ exactly (to within any desired precision).

\input{a_more_than_b}

\input{analytical}

\input{numerical}

%% file: a_more_than_b.tex
\subsection{Proving $A_\alpha \geq B_\alpha$}\label{ss:a_geq_b}

Looking ahead, $A_\alpha$ admits decomposition into a finite sum or a convergent series. By comparison, manipulating $B_\alpha$ is a similar manner is considerably more difficult, since we have a composite term $\mu$ in the denominator. Fortunately, it is not necessary as we demonstrate that $B_\alpha \leq A_\alpha$. In fact, we prove a more general statement, which may be of independent interest.

\begin{thm}\label{th:general_bound}Let $P$ and $Q$ be two differentiable distributions on $\mathcal{X}$ such that there exists a differentiable mapping $\nu\colon\mathcal{X}\mapsto\mathcal{X}$ satisfying $\nu(\nu(x))=x$ and $P(x)=Q(\nu(x))$. Then the following holds for all $\alpha\geq 1$ and $q\in[0,1]$:
\[
\rdalpha{(1-q)P+qQ}{Q}\geq \rdalpha{Q}{(1-q)P+qQ}.
\]
\end{thm}
\begin{proof}We will repeatedly use the fact that $Q(x)=Q(\nu(\nu(x))=P(\nu(x))$. Furthermore, since the inverse of $\nu$ is $\nu$, it is continuously differentiable and its Jacobian satisfies $\det \mathbf{J}_\nu=\pm 1$.

Let $P_q\eqdef (1-q)P+qQ$ and $Q_q\eqdef (1-q)Q+qP$.  Then, substituting $x=\nu(y)$, we have
\begin{multline*}
\rdalpha{P_q}{Q}=\int_\mathcal{X} Q(x)\left(\frac{P_q(x)}{Q(x)}\right)^{\alpha}\ud x=\int_{\nu^{-1}(\mathcal{X})}Q(\nu(y))\left(\frac{P_q(\nu(y))}{Q(\nu(y))}\right)^{\alpha}|\det \mathbf{J}_\nu(y)|\ud y\\
=\int_\mathcal{X} P(y)\left(\frac{Q_q(y)}{P(y)}\right)^{\alpha}\ud y
\end{multline*}
and similarly
\[
\rdalpha{Q}{P_q}=\int_\mathcal{X} P_q(x)\left(\frac{Q(x)}{P_q(x)}\right)^{\alpha}\ud x=\int_\mathcal{X} Q_q(y)\left(\frac{P(y)}{Q_q(y)}\right)^{\alpha}\ud y.
\]
We may now express
\begin{align}
\rdalpha{P_q}{Q}&=\frac12\left\{\int_\mathcal{X} Q(x)\left(\frac{P_q(x)}{Q(x)}\right)^{\alpha}\ud x+\int_\mathcal{X} P(y)\left(\frac{Q_q(y)}{P(y)}\right)^{\alpha}\ud y\right\}\nonumber\\
&=\frac12\int_\mathcal{X}\left\{Q(x)\left(\frac{P_q(x)}{Q(x)}\right)^{\alpha}+P(x)\left(\frac{Q_q(x)}{P(x)}\right)^{\alpha}\right\}\ud x\label{eq:int_A}\\
\intertext{and}
\rdalpha{Q}{P_q}
&=\frac12\int_\mathcal{X}\left\{ P_q(x)\left(\frac{Q(x)}{P_q(x)}\right)^{\alpha} + Q_q(x)\left(\frac{P(x)}{Q_q(x)}\right)^{\alpha}\right\}\ud x.\label{eq:int_B}
\end{align}
We proceed by arguing a stronger statement: the integrand of the right-hand side of~(\ref{eq:int_A}) dominates the integrand of~(\ref{eq:int_B}) \emph{pointwise}. In other words, we prove the following lemma, from which the theorem claim follows:

\begin{lem}\label{lem:a>=b}For all $x\in\mathcal{X}$, and any $\alpha>1$, $q\in [0,1]$:
\[	Q(x)\left(\frac{P_q(x)}{Q(x)}\right)^{\alpha}+P(x)\left(\frac{Q_q(x)}{P(x)}\right)^{\alpha}\geq P_q(x)\left(\frac{Q(x)}{P_q(x)}\right)^{\alpha} +Q_q(x)\left(\frac{P(x)}{Q_q(x)}\right)^{\alpha},
\]
where $P_q(x)=(1-q)P(x)+qQ(x)$ and $Q_q(x)=(1-q)Q(x)+qP(x)$.
\end{lem}
\begin{proof}Let $u\eqdef P(x)$ and $v\eqdef Q(x)$. Assume wlog $u\geq v$ (the claim is symmetric with respect to $P$ and $Q$). Then the two sides become, respectively,
	\[v\left((1-q)+q\frac{u}{v}\right)^\alpha+u\left((1-q)+q\frac{v}{u}\right)^\alpha\quad\textrm{and }\quad
	v\left((1-q)+q\frac{u}{v}\right)^{1-\alpha}+u\left((1-q)+q\frac{v}{u}\right)^{1-\alpha}.\]

	Dividing by $v$ and letting $y \eqdef (1-q)+q\frac{u}{v}$ and $z \eqdef (1-q)+q\frac{v}{u}$,  we need to compare
	\[
	y^\alpha+(y - 1) / (1-z)z^\alpha\quad\textrm{ and }\quad y^{1-\alpha}+(y - 1) / (1-z)z^{1-\alpha}
	\]
	subject to $y \geq 1/z  \geq 1\geq z$. (The bound $y\geq 1$ follows from $u\geq v$ and $y\geq 1/z$ from $yz=(1-q)^2+q(1-q)(\frac{u}{v}+\frac{v}{u})+q^2\geq (1-q)^2+2q(1-q)+q^2=1$.)

	Collecting the $y$ terms on the left and the $z$ terms on the right we have
	\[
	\frac{y^\alpha-y^{1-\alpha}}{y-1}\quad\textrm{ and }\quad\frac{z^{1-\alpha}-z^{\alpha}}{1-z}.
	\]

	The left-hand side dominates, since the two expressions are equal for $y=1/z$ and the left expression is monotonically increasing in $y$ over $[1,\infty)$:
	\begin{multline*}
	\left(\frac{y^\alpha-y^{1-\alpha}}{y-1}\right)'=\frac{(\alpha y^{\alpha-1}-(1-\alpha)y^{-\alpha})(y-1)-y^\alpha+y^{1-\alpha}}{(y-1)^2}\\
	=\frac{(\alpha-1)(y^\alpha-y^{-\alpha})-\alpha(y^{\alpha-1}-y^{1-\alpha})}{(y-1)^2}\geq 0,
	\end{multline*}
	which holds due to $(y^\alpha-y^{-\alpha})/\alpha\geq (y^{\alpha-1}-y^{1-\alpha})/(\alpha-1)$,  in turn implied by monotonicity of $\sinh(\alpha\ln y)/\alpha$ in $\alpha$ for $y\geq 1$.
\end{proof}
This concludes the proofs of the lemma and of the theorem.
\end{proof}

\begin{cor}$A_\alpha\geq B_\alpha$ for any $\alpha\geq 1$.
\end{cor}
\begin{proof}To apply Theorem~\ref{th:general_bound} it suffices to find a mapping $\nu\colon \R\mapsto\R$ so that $\mu_0(x)=\mu_1(\nu(x))$ and $\nu=\nu^{-1}$.

Let $\nu(x)\eqdef 1-x$. Then, the pdf of $\mu_0$ is $\propto \exp(-x^2/2\sigma^2)=\exp(-(\nu(x)-1)^2/2\sigma^2)$ as required. The claim follows.
\end{proof}

Theorem~\ref{th:general_bound} holds for any additive noise whose distribution is centrally symmetric, which also includes Laplace, sinh-normal~\cite{tCDP}, their discretized and multi-dimensional variants.

%% file: analytical.tex
\subsection{Closed-Form Bound}\label{ss:closed-form}

We write $A_\alpha$ as an integral over the real line, break it into two parts (at $z_0$ to be chosen shortly) and bound them separately as follows:
\begin{align*}
A_\alpha&=\int_{-\infty}^\infty \mu_0(z)\left( (1-q) + q \frac{\mu_1(z)}{\mu_0(z)}\right)^\alpha \ud z\\
&=\int_{-\infty}^{z_0} \mu_0(z)\left( (1-q) + q \frac{\mu_1(z)}{\mu_0(z)}\right)^\alpha \ud z\tag{$\eqdef A_\alpha^{(1)}$}\\
&\hphantom{=}+\int_{z_0}^\infty \mu_0(z)\left( (1-q) + q \frac{\mu_1(z)}{\mu_0(z)}\right)^\alpha \ud z.\tag{$\eqdef A_\alpha^{(2)}$}
\end{align*}

We define $z_0\eqdef \frac12+\sigma^2\ln\left(1+\frac1{q(\alpha-1)}\right)$, chosen to satisfy $(1-q) + q \frac{\mu_1(z_0)}{\mu_0(z_0)}=\alpha/(\alpha - 1)$. For notational convenience we also introduce $r_0\eqdef\frac{\mu_1(z_0)}{\mu_0(z_0)}=1+\frac1{q(\alpha-1)}$. Note that $z_0>\frac12$ and $r_0>1$. We repeatedly use the facts that $\mu_0(z)=\mu_1(1-z)$ everywhere and the ratio $\mu_0(z)/\mu_1(z)$ is monotonically decreasing in $z$. It follows that for all $z\leq z_0$ the ratio $\mu_1(z)/\mu_0(z)\leq r_0$ and for $z\geq 1-z_0$ the ratio $\mu_0(z)/\mu_1(z)\leq r_0$. In particular, for $z\in[1-z_0,z_0]$ the  ratios $\mu_1(z)/\mu_0(z)$ and $\mu_1(z)/\mu_0(z)$ are confined to $[1/r_0,r_0]$.

\begin{lem}\label{lem:symmetric}For any $\alpha\geq 1$, $q\in[0,1]$ and positive $x, y$ such that $1/r_0 \leq x/y\leq r_0$ where $r_0=1+\frac1{q(\alpha-1)}$:
	\[
	x \cdot \left(1 + q \frac{y-x}{x}\right)^\alpha + y \cdot \left(1 + q\frac{x-y}y\right)^\alpha\leq\\
	(x + y) + q^2 \alpha (\alpha - 1) \left(\frac{x^2}y + \frac{y^2}x - (x + y)\right).
	\]
\end{lem}

\begin{proof}
Recall the following Lemma 15 in Bun et al.~\cite{tCDP}:
\[
(1+w)^\alpha \leq 1 + \alpha w + \alpha (\alpha-1)w^2\quad\text{if }-1<w\leq 1/(\alpha-1)\quad\text{and}\quad\alpha \geq 1.
\]

Let $r=x/y\in [1/r_0,r_0]$. By setting $w\eqdef q(1/r-1)$ and $w\eqdef q(r-1)$ we obtain, respectively,
\begin{align*}
r\cdot (1 + q\cdot (1 / r - 1))^\alpha &\leq r + \alpha q (1 - r) + q^2\alpha(\alpha - 1)\cdot (1/r - 2 + r)
\intertext{and}
(1 + q (r - 1))^\alpha & \leq 1 + \alpha q (r - 1) + q^2 \alpha (\alpha - 1) (r^2  - 2r + 1).
\end{align*}
The claim follows by simple addition and multiplication of both sides by $y$.
\end{proof}

A useful fact that facilitates application of Lemma~\ref{lem:symmetric} is that for positive $x$ and $y$
\begin{equation}\label{eq:nonnegative}
\frac{x^2}y + \frac{y^2}x - (x + y)=\left(\frac{x}{y}-1\right)^2\left(y+\frac{y^2}{x}\right)\geq 0.
\end{equation}
It implies that all terms in the claim of the lemma are non-negative.

\begin{lem}\label{lem:a1}For $\alpha\geq 1$ and $q\in[0,1]$
\[
A_\alpha^{(1)}\leq 1+q^2 \alpha (\alpha - 1)\left(\exp(1/\sigma^2)-1\right).
\]
\end{lem}

\begin{proof}Similarly to the argument in the previous section, we ``double over'' the integral, by using the symmetry between $\mu_0$ and $\mu_1$:
\begin{align*}
A_\alpha^{(1)}&=\int_{-\infty}^{z_0} \mu_0(z)\left( (1-q) + q \frac{\mu_1(z)}{\mu_0(z)}\right)^\alpha\ud z\\
&=\frac12\left\{\int_{-\infty}^{z_0} \mu_0(z)\left( (1-q) + q \frac{\mu_1(z)}{\mu_0(z)}\right)^\alpha\ud z+\int_{1-z_0}^\infty \mu_1(z)\left( (1-q) + q \frac{\mu_0(z)}{\mu_1(z)}\right)^\alpha\ud z\right\}\\
&=\hphantom{+}\frac12\int_{-\infty}^{1-z_0} \mu_0(z)\left((1-q) + q \frac{\mu_1(z)}{\mu_0(z)}\right)^\alpha\ud z\\
&\hphantom{=}+\frac12\int_{1-z_0}^{z_0}\left\{\mu_0(z)\left( (1-q) + q \frac{\mu_1(z)}{\mu_0(z)}\right)^\alpha + \mu_1(z)\left( (1-q) + q \frac{\mu_0(z)}{\mu_1(z)}\right)^\alpha\right\}\ud z\\
&\hphantom{=}+\frac12\int_{z_0}^\infty \mu_1(z)\left( (1-q) + q \frac{\mu_0(z)}{\mu_1(z)}\right)^\alpha\ud z.
\end{align*}
Note that $(1-q) + q \frac{\mu_1(z)}{\mu_0(z)}< 1$ for $z\leq 1-z_0$ and $(1-q) + q \frac{\mu_0(z)}{\mu_1(z)}< 1$ for $z\geq z_0$. Furthermore, $1/r_0\leq \frac{\mu_1(z)}{\mu_0(z)}\leq r_0$ for $z\in [1-z_0,z_0]$.

Applying Lemma~\ref{lem:symmetric} and inequality~(\ref{eq:nonnegative}), we bound $A_\alpha^{(1)}$:
\begin{multline*}
A_\alpha^{(1)}\leq \frac12\int_{-\infty}^\infty \mu_0(z) + \mu_1(z) + q^2 \alpha (\alpha - 1) \left(\frac{\mu_0(z)^2}{\mu_1(z)} + \frac{\mu_1(z)^2}{\mu_0(z)} - (\mu_0(z) + \mu_1(z))\right)\ud z=\\
1+q^2 \alpha (\alpha - 1)\left(\exp(1/\sigma^2)-1\right)
\end{multline*}
as claimed (we use the identities $\int_{-\infty}^\infty \frac{\mu_0(z)^2}{\mu_1(z)}\ud z=\int_{-\infty}^\infty \frac{\mu_1(z)^2}{\mu_0(z)}\ud z=\exp(1/\sigma^2)$ elaborated in Section~\ref{ss:numerical}).
\end{proof}

Lemma~\ref{lem:a1} is unconditional, mildly behaved, and it yields an asymptotically right dependency of the bound on the parameters $q$, $\alpha$, and $\sigma$ (after taking the logarithm, it is non-tight by a factor close to 2). We next consider the integral to the right of $z_0$ that is responsible for the conditions on $\alpha$.

\begin{lem}\label{lem:a2}
If $q\leq \frac15$, $\sigma\geq 4$, and $\alpha$ satisfy
\begin{align}
1<\alpha&\leq {\textstyle\frac12}\sigma^2L-2\ln \sigma,\label{eq:alpha1}\\
\alpha&\leq \frac{\frac12\sigma^2 L^2-\ln5-2\ln\sigma}{L+\ln (q\alpha)+1/(2\sigma^2)},\label{eq:alpha2}
\end{align}
where $L\eqdef\ln\left(1+\frac1{q(\alpha-1)}\right)$, then
\[
A_\alpha^{(2)}\leq 0.9\cdot q^2 \alpha (\alpha-1)/\sigma^2.
\]
\end{lem}
\begin{proof}We observe (similarly to~\cite[Lemma 16]{tCDP}) that for $z\geq z_0$, it holds that
\[
(1-q) + q \frac{\mu_1(z)}{\mu_0(z)}\leq q \alpha \frac{\mu_1(z)}{\mu_0(z)}.
\]
Thus,
\begin{align}
A_\alpha^{(2)}&=\int_{z_0}^\infty \mu_0(z)\left( (1-q) + q \frac{\mu_1(z)}{\mu_0(z)}\right)^\alpha\ud z\nonumber\\
&\leq \int_{z_0}^\infty \mu_0(z)\left(q\alpha\frac{\mu_1(z)}{\mu_0(z)}\right)^\alpha\ud z\nonumber\\
&=(q\alpha)^\alpha \frac{1}{\sigma\sqrt{2\pi}}\int_{z_0}^{\infty} \exp\left\{-\frac{z^2}{2\sigma^2} + \alpha \cdot \frac{z^2-(z-1)^2}{2\sigma^2}\right\} \ud z\nonumber\\
&=(q\alpha)^\alpha \exp\left(\frac{\alpha^2 - \alpha}{2\sigma^2}\right)  \frac{1}{\sigma\sqrt{2\pi}}\int_{z_0-\alpha}^{\infty} \exp\left(-\frac{y^2}{2\sigma^2}\right)\ud y\tag{\textrm{substituting $y\eqdef z-\alpha$}}\\
&\leq(q\alpha)^\alpha\exp\left(\frac{\alpha^2-\alpha-(z_0-\alpha)^2}{2\sigma^2}\right)\label{eq:tail}\\
&<(q\alpha)^\alpha\exp\left(\frac{-z_0(z_0-2\alpha)}{2\sigma^2}\right).\label{eq:A2_cases}
\end{align}
The inequality (\ref{eq:tail}) follows from the Gaussian tail bound $\int_t^\infty \mu_0(x)\ud x<\exp(-t^2/(2\sigma^2))$ for $t\geq 0$. To apply the bound it is necessary to check that the lower limit of the integral $t=z_0-\alpha\geq 0$. Recall that $z_0=\frac12+\sigma^2\ln\left(1+\frac1{q(\alpha-1)}\right)=\frac12+\sigma^2L$. Together, the definition of $z_0$ and condition~(\ref{eq:alpha1}) imply that
\begin{equation}\label{eq:z0bound}
z_0\geq 2\alpha+4\ln\sigma,
\end{equation}
which, given that $\sigma>1$, guarantees that $z_0>2\alpha$.

It is convenient to take the logarithm of both sides of Eq.~(\ref{eq:A2_cases}) resulting in the following:
\begin{align*}
\ln A_\alpha^{(2)}&<\alpha\ln (q\alpha)-\frac{z_0(z_0-2\alpha)}{2\sigma^2}\\
&<\alpha\ln (q\alpha)-\frac{2z_0\ln\sigma}{\sigma^2}\tag{by Eq.~(\ref{eq:z0bound})}\\
&<\alpha(\ln q+\ln\alpha)-2L\ln \sigma.
\end{align*}
We will argue that the right-hand side of the above is less than $\ln(0.9\cdot q^2\alpha(\alpha-1)/\sigma^2)$. Subtracting the two quantities, we need to demonstrate that
\begin{equation}\label{eq:ln}
(2-\alpha)\ln q+(1-\alpha)\ln \alpha+\ln(\alpha-1)+2(L-1)\ln \sigma>-\ln(0.9).
\end{equation}

The rest of the proof proceeds by case analysis.

\paragraph{Case I: $\alpha<2$} Observing that $\ln q + \ln(\alpha-1)+L=\ln(q(\alpha-1)+1)>0$, it suffices to verify that
\begin{equation}\label{eq:caseI}
(1-\alpha)(\ln q+\ln \alpha)+(L-1)(2\ln \sigma-1)>1-\ln(0.9).
\end{equation}
From $q<1/5$ and $1<\alpha<2<1/q$ we conclude that the first summand is positive. Since $L>\ln(1+\frac1q)\geq \ln 6$ and $\sigma\geq 4$, the second summand is larger than $1-\ln(0.9)$, by direct computation.

We now consider the following two cases.

\paragraph{Case IIa: $\alpha\geq 2$ and $q(\alpha-1)<1/3$} Continuing Eq.~(\ref{eq:ln}), we have
\begin{multline*}
(2-\alpha)\ln q+(1-\alpha)\ln \alpha+\ln(\alpha-1)+2(L-1)\ln \sigma= \\(2-\alpha)(\ln q+\ln \alpha)-\ln\alpha+\ln(\alpha-1)+2(L-1)\ln \sigma\geq -\ln(0.9).
\end{multline*}
The inequality holds since  the first term is non-negative ($q\alpha = q(\alpha-1)\cdot \frac{\alpha}{\alpha-1}<2/3$), $\ln(\alpha-1)-\ln\alpha\geq -\ln2$,  $L> \ln 4$, and  $\sigma\geq 4$.

\paragraph{Case IIb: $\alpha\geq 2$ and $q(\alpha-1)\geq 1/3$} Continue Eq.~(\ref{eq:A2_cases}), taking the logarithm of both sides:
\begin{align*}
\ln A_\alpha^{(2)}&<\alpha \ln(q\alpha) - \frac{z_0^2}{2\sigma^2}+\frac{\alpha z_0}{\sigma^2}\\
&<\alpha\left(\ln(q\alpha)+L+\frac1{2\sigma^2}\right)-\frac{L}2 - \frac{\sigma^2L^2}{2}\tag{by $z_0=\frac12+\sigma^2L$}\\
&\leq-\ln5 - 2\ln \sigma - \frac{L}2.
\end{align*}
The last inequality follows from Eq.~(\ref{eq:alpha2}) and by the fact that $L+\ln(q\alpha)=\ln(q\alpha+\frac{\alpha}{\alpha-1})>0$.

Exponentiating both sides, we have
\begin{align*}
A_\alpha^{(2)}&< \frac1{5\sigma^2\sqrt{1+\frac1{q(\alpha-1)}}}\\
&\leq \frac3{10}\cdot\frac{q(\alpha-1)}{\sigma^2}\tag{$\sqrt{1+x}\geq\frac23x$ for $x\leq 3$, and plugging in $x\eqdef\frac1{q(\alpha-1)}$}\\
&< \frac9{10}\frac{q^2\alpha(\alpha-1)}{\sigma^2}\tag{by $q\alpha> q(\alpha-1)\geq 1/3$}
\end{align*}
as needed.
\end{proof}

Our main theorem holds under the conditions of Lemmas~\ref{lem:a1} and~\ref{lem:a2}:
\begin{thm}\label{thm:sgm}If $q\leq \frac15$, $\sigma\geq 4$, and $\alpha$ satisfy
\begin{align*}
1<\alpha&\leq {\textstyle\frac12}\sigma^2L-2\ln \sigma,\\
\alpha&\leq \frac{\frac12\sigma^2 L^2-\ln5-2\ln\sigma}{L+\ln (q\alpha)+1/(2\sigma^2)},
\end{align*}
where $L=\ln\left(1+\frac1{q(\alpha-1)}\right)$, then SGM applied to a function of $\ell_2$-sensitivity 1 satisfies $(\alpha, \eps)$-RDP where
\[
\eps\eqdef 2q^2 \alpha/\sigma^2.
\]
\end{thm}
\begin{proof}Recall that to state an RDP guarantee on SGM it is sufficient to bound $A_\alpha$. Applying the results of Lemmas~\ref{lem:a1} and~\ref{lem:a2},
\begin{multline*}
A_\alpha=A_\alpha^{(1)}+A_\alpha^{(2)}\leq 1+q^2 \alpha (\alpha - 1)\left(\exp(1/\sigma^2)-1\right) + 0.9\cdot q^2 \alpha (\alpha-1)/\sigma^2\leq\\
1+2q^2 \alpha (\alpha-1)/\sigma^2.
\end{multline*}
(The last inequality due to $\exp(1/\sigma^2)-1\leq 1.1/\sigma^2$ for $\sigma\geq 4$.)

Finally, since $\ln(1+x)<x$ for $x\geq 0$, we conclude that SGM satisfies $(\alpha, \eps)$-RDP where
$\eps=\frac1{\alpha-1}\ln A_\alpha\leq 2q^2 \alpha/\sigma^2$ as needed.
\end{proof}

%% file: numerical.tex
\subsection{Numerically Stable Computation}\label{ss:numerical}

Na\"ively, $A_\alpha$ can be approximated as an integral using standard numerical libraries. It however leads to the problem of computing an integral over the whole real line of a quantity that can vary a lot. We sidestep this difficulty by expressing $A_\alpha$ as a finite sum (or a convergent series), swap the order of the integration and summation operators, and compute the integrals analytically. 

To compute $A_\alpha = \E_{z \sim \mu_0} [ (\mu(z) / \mu_0(z))^\alpha]$, we write
\begin{equation}\label{eq:mu_over_mu0}
\left(\frac{\mu(z)}{\mu_0(z)}\right)^\alpha = \left( (1-q) + q \frac{\mu_1(z)}{\mu_0(z)}\right)^\alpha
\end{equation}
and consider two cases.

\paragraph{Case I: Integer $\alpha$.} Applying the binomial expansion to (\ref{eq:mu_over_mu0}), we have
 \[
 \left(\frac{\mu(z)}{\mu_0(z)}\right)^\alpha= \sum_{k=0}^{\alpha} {\alpha \choose k}  (1-q)^{\alpha-k}q^{k} \left(\frac{\mu_1(z)}{\mu_0(z)}\right)^k.
\]
Thus it suffices to compute for $k \in \{0,\dots,\alpha\}$ the expectation
\[
\E_{z \sim \mu_0}\left[\left(\frac{\mu_1(z)}{\mu_0(z)}\right)^k\right].
\]
We observe that the terms of the form $\E_{z \sim \mu_0} \left[\left(\frac{\mu_1(z)}{\mu_0(z)}\right)^k\right]$ have an analytical closed form that can be obtained by integration. Indeed,
\begin{align*}
  \E_{z \sim \mu_0} \left[\left(\frac{\mu_1(z)}{\mu_0(z)}\right)^k\right]
  &= \frac{1}{\sigma\sqrt{2\pi}}\int_{-\infty}^{\infty} \exp\left\{-\frac{x^2}{2\sigma^2} + k \cdot \frac{x^2-(x-1)^2}{2\sigma^2}\right\} \ud x\\
  &= \frac{1}{\sigma\sqrt{2\pi}}\int_{-\infty}^{\infty} \exp\left\{-\frac{x^2}{2\sigma^2} + \frac{2kx - k}{2\sigma^2}\right\} \ud x\\
  &= \frac{1}{\sigma\sqrt{2\pi}}\int_{-\infty}^{\infty} \exp\left\{-\frac{(x-k)^2}{2\sigma^2} + \frac{ k^2-k}{2\sigma^2}\right\} \ud x\\
  &= \exp\left(\frac{k^2 - k}{2\sigma^2}\right)  \frac{1}{\sigma\sqrt{2\pi}}\int_{-\infty}^{\infty} \exp\left(-\frac{y^2}{2\sigma^2}\right)\ud y\tag{\textrm{substituting $y\eqdef x-k$}}\\
  &= \exp\left(\frac{k^2 - k}{2\sigma^2}\right).
\end{align*}

\paragraph{Case II. Fractional $\alpha$.} To rewrite  (\ref{eq:mu_over_mu0}) as a convergent series, consider two cases depending on how $1-q$ compares with $q\mu_1(z)/\mu_0(z)$. The inflection point is $z_1$ where the two quantities are equal:
\begin{align*}
(1-q)\mu_0(z_1)=q\mu_1(z_1) &\myiff  1-q=q e^{\frac{2z_1-1}{2\sigma^2}}\\
&\myiff z_1=\frac12+\sigma^2\ln(q^{-1} - 1).
\end{align*}
Thus we can express 
\begin{equation}
\left(\frac{\mu(z)}{\mu_0(z)}\right)^\alpha=\begin{cases} \sum_{k=0}^\infty {\alpha \choose k}  (1-q)^{\alpha-k}q^{k} \left(\frac{\mu_1(z)}{\mu_0(z)}\right)^k\qquad\text{when } z\leq z_1\\
\sum_{k=0}^\infty {\alpha \choose k}  (1-q)^k q^{\alpha-k} \left(\frac{\mu_1(z)}{\mu_0(z)}\right)^{\alpha-k}\quad\text{when } z> z_1\end{cases}.\label{eq:series}
\end{equation}

Analogously to the case of integer $\alpha$, we compute the expectations of both series under $z\sim \mu_0$, where the integrals are taken over the half lines $(-\infty, z_1]$ and $[z_1, +\infty)$:
\begin{align*}
\int_{-\infty}^{z_1}\mu_0(x) \left(\frac{\mu_1(x)}{\mu_0(x)}\right)^k\ud x
&= \frac{1}{\sigma\sqrt{2\pi}}\int_{-\infty}^{z_1} \exp\left\{-\frac{x^2}{2\sigma^2} + k \cdot \frac{x^2-(x-1)^2}{2\sigma^2}\right\} \ud x\\
&=\exp\left(\frac{k^2 - k}{2\sigma^2}\right)  \frac{1}{\sigma\sqrt{2\pi}}\int_{-\infty}^{z_1-k} \exp\left(-\frac{y^2}{2\sigma^2}\right)\ud y\\
&=\frac12\exp\left(\frac{k^2 - k}{2\sigma^2}\right)\erfc\left(\frac{k-z_1}{\sqrt{2}\sigma}\right),\\
\int_{z_1}^{\infty}\mu_0(x) \left(\frac{\mu_1(x)}{\mu_0(x)}\right)^k\ud x
&= \frac12\exp\left(\frac{k^2 - k}{2\sigma^2}\right)\erfc\left(\frac{z_1-k}{\sqrt{2}\sigma}\right).
\end{align*}

The computation done in the privacy accountant proceeds by plugging  in these quantities into the series~(\ref{eq:series}), and carrying out the summation to convergence.

Upper bounds (Theorem \ref{thm:sgm}) and exact computations are compared in~\autoref{fig:two_plots}.

\begin{figure}
	\centering
	\begin{subfigure}{.45\textwidth}
	\centering
	\includegraphics[width=1.\linewidth]{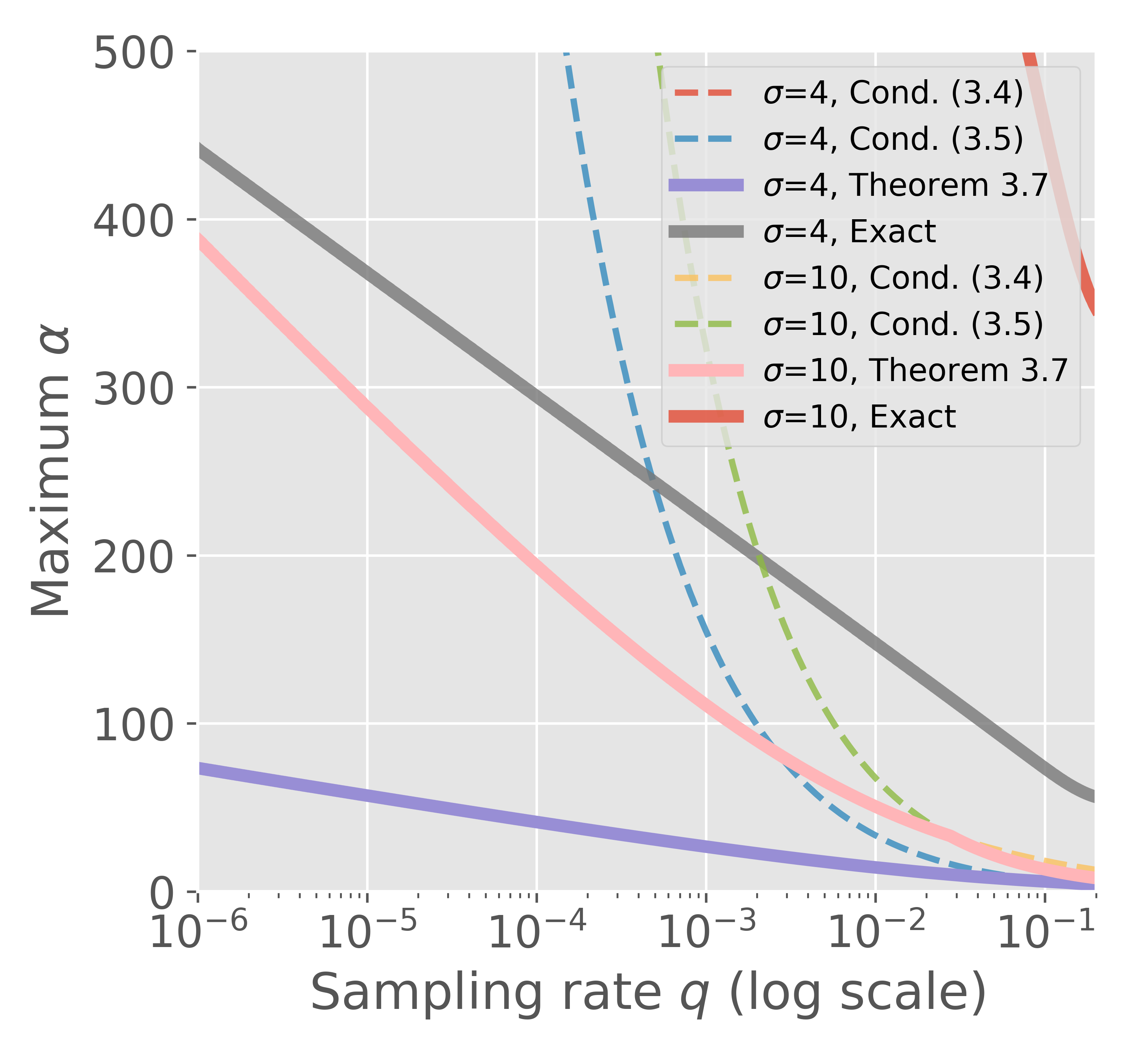}
\end{subfigure}
\begin{subfigure}{.45\textwidth}
	\centering
	\includegraphics[width=1.\linewidth]{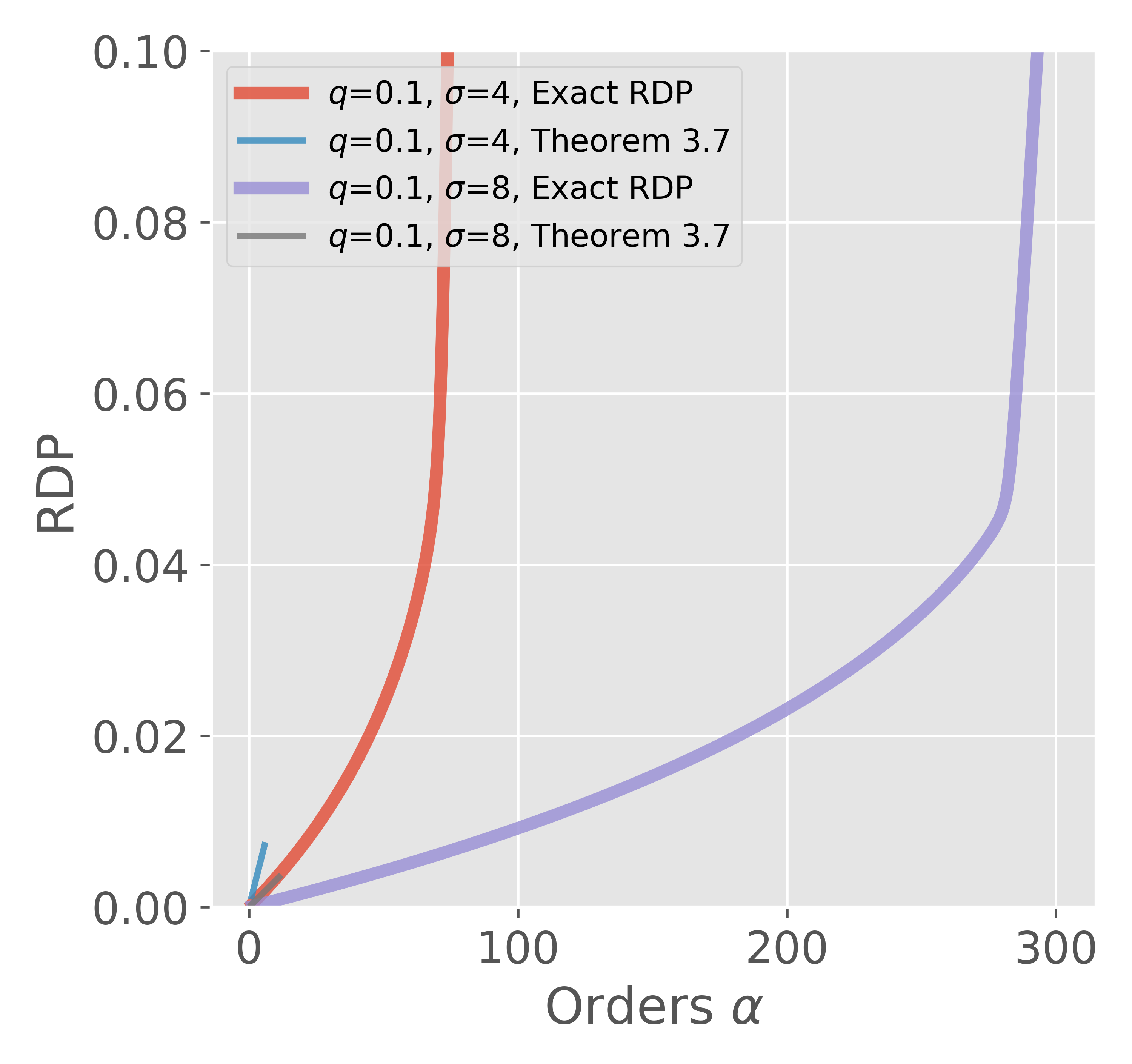}
\end{subfigure}
\caption{\textbf{Left:} Maximum $\alpha$ as a function of $q$ for $\sigma=4, 10$. For each $\sigma$ four graphs are plotted: bounds due to conditions~\ref{eq:alpha1} and~\ref{eq:alpha2}, their minimum, and the exact $\alpha$ so that $\eps=2q^2\alpha/\sigma^2$. \textbf{Right:} RDP of $\SG_{q,\sigma}$ computed exactly (Section~\ref{ss:numerical}) and bounded according to Theorem~\ref{thm:sgm}.}
\label{fig:two_plots}\end{figure}

%% file: discussion.tex
\section{Discussion}\label{s:discussion}

The notions of CDP, zCDP, Moments accountant and \renyi\ DP are closely related in that they control the moments of the privacy loss random variable. In this section, we clarify their differences.

For two adjacent datasets $S$ and $S'$, the privacy loss of  a mechanism $\M$ at an outcome $z$ is defined as 
\begin{equation}\label{eq:privacyloss}
  c(z; \M,  S, S') \eqdef \log \frac{\Pr[\M(S) = z]}{\Pr[\M(S') = z]}.
\end{equation}

For continuous output spaces, the probability above is replaced by the probability density function. The various definitions deal with moment generating function of the privacy loss random variable. Define
\begin{align*}
M_{\alpha}(\M, S, S') &\eqdef \E_{z \sim \M(S)}[\exp(\alpha\cdot c(z; \M, S, S'))]\\
	&= \E_{z \sim \M(S)}[ \left[\left(\frac{\Pr[\M(S) = z]}{\Pr[\M(S') = z]}\right)^{\alpha}\right]\\
	&= \E_{z \sim \M(S')}\left[\left(\frac{\Pr[\M(S) = z]}{\Pr[\M(S') = z]}\right)^{\alpha+1}\right].
\end{align*}
Recalling the definition of \renyi\ divergence between distributions:
\[
	\rdalpha{P}{Q} = \frac{1}{\alpha-1}\log \E_{z \sim Q} \left(\frac{P(z)}{Q(z)}\right)^{\alpha},
\]
it follows that
\begin{equation}
	M_{\alpha}(\M, S, S') = \exp(\alpha\cdot \rd{\alpha+1}{\M(S')}{\M(S)}.\label{eq:mvsd}
\end{equation}
The following are equivalent definitions of CDP, zCDP, tCDP and RDP:
\begin{defi}
A mechanism $\M$ satisfies $(\mu, \tau)$-CDP if for all adjacent datasets $S, S'$ and for all $\alpha \geq 1$
\begin{align*}
	\E_{z \sim \M(S)}[c(z; \M, S, S')] &\leq \mu,\\
	M_{\alpha}(\M, S, S') &\leq \exp\left(\frac{\alpha^2\beta^2}{2}\right).
\end{align*}
A mechanism $\M$ satisfies $(\xi, \rho)$-zCDP if for all adjacent datasets $S, S'$,
\[
	\forall \alpha > 0\colon M_{\alpha}(\M, S, S') \leq  \exp(\alpha(\xi + \rho(\alpha+1)).
\]
A shorthand for $(0, \rho)$-zCDP is $\rho$-zCDP. Thus a mechanism $\M$ satisfies $\rho$-zCDP if for all adjacent datasets $S, S'$,
\[
	\forall \alpha > 0\colon M_{\alpha}(\M, S, S') \leq  \exp(\rho\alpha(\alpha+1)).
\]
A mechanism $\M$ satisfies $(\rho, \omega)$-tCDP if for all adjacent datasets $S, S'$,
\[
	\forall \alpha \in (0, \omega)\colon M_{\alpha}(\M, S, S') \leq \exp(\rho\alpha(\alpha + 1)).
\]
A mechanism $\M$ satisfies $(\alpha, \eps)$-RDP if for all adjacent datasets $S, S'$,
\[
	M_{\alpha-1}(\M, S, S') \leq \exp((\alpha-1)\eps).
\]
\end{defi}
While here we have stated the definitions in terms of $M_{\alpha}$, using (\ref{eq:mvsd}), one can translate these to bounds on the \renyi\ divergence; in some cases that leads to cleaner looking definitions. Going down the list, the definitions get less restrictive and have more parameters. While zCDP suffices for many purposes, SGM is an important mechanism that does not satisfy zCDP, but satisfies RDP for a suitable range of $\alpha$.

While the above were proposed as standalone privacy definitions, the moments accountant was proposed as an accounting mechanism that tracks (the logarithm of) $M_{\alpha}$ directly and converts the resulting bound to an $(\eps, \delta)$-DP bound.